\documentclass[a4paper,fleqn]{cas-sc}

\usepackage[authoryear,longnamesfirst]{natbib}

\usepackage{subfigure}

\usepackage{graphicx}
\usepackage{amssymb}
\usepackage{amsfonts,mathtools}
\usepackage{amsmath}
\usepackage{amsthm}
\newtheorem{theorem}{Theorem}
\newtheorem{corollary}{Corollary}[theorem]

\usepackage{longtable,tabularx}

\usepackage{cleveref}

\usepackage[table,dvipsnames]{xcolor}

\usepackage[listings,skins,breakable]{tcolorbox}
{%
\end{tcolorbox}\endgroup}

{%
\end{tcolorbox}\endgroup}

\newtcbox{\popovnotesmall}{breakable,enhanced jigsaw,nobeforeafter,tcbox raise base,boxrule=0.4pt,top=0mm,bottom=0mm,
  right=0mm,left=4mm,arc=1pt,boxsep=2pt,before upper={\vphantom{dlg}},
  colframe=Fuchsia!75!black,coltext=black,colback=Fuchsia!20,
  overlay={\begin{tcbclipinterior}\fill[Fuchsia!75!black] (frame.south west)
    rectangle node[text=white,font=\sffamily\bfseries\tiny,rotate=90] {AAP} ([xshift=4mm]frame.north west);\end{tcbclipinterior}}}
\MakeRobust\popovnotesmall
\ifdefined
\fi

{%
\end{tcolorbox}\endgroup}

\def\*#1{\boldsymbol{\mathbf{#1}}}
\def\##1{\bm{\mathsf{#1}}}

\usepackage{newunicodechar,graphicx}
\usepackage{svg}

\DeclareRobustCommand{\okina}{%
  \raisebox{\dimexpr\fontcharht\font`A-\height}{%
    \scalebox{0.8}{`}%
  }%
}
\newunicodechar{ʻ}{\okina}
\usetikzlibrary{positioning}

\usepackage{tikz}
\usetikzlibrary{calc}
\usetikzlibrary{math}

\DeclarePairedDelimiterX{\norm}[1]{\lVert}{\rVert}{#1}

\newcommand{\appropto}{\mathrel{\vcenter{
  \offinterlineskip\halign{\hfil$##$\cr
    \propto\cr\noalign{\kern2pt}\sim\cr\noalign{\kern-2pt}}}}}

\usepackage{pgfplots}
\pgfplotsset{compat=1.18} 
\usepackage{pgfplotstable}
\usepackage{cbcolors}

\newtheorem{remark}{Remark}

\def\tsc#1{\csdef{#1}{\textsc{\lowercase{#1}}\xspace}}
\tsc{WGM}
\tsc{QE}
\tsc{EP}
\tsc{PMS}
\tsc{BEC}
\tsc{DE}
\begin{document}
\let\WriteBookmarks\relax
\def\floatpagepagefraction{1}
\def\textpagefraction{.001}
\shorttitle{Learning Discriminators for EnGMF}
\shortauthors{Z. Jabbar et~al.}
%\begin{frontmatter}

\title [mode = title]{Learning Discriminators for Resampling in the Ensemble Gaussian Mixture Filter through a Normalizing Flow Approach}                      
\tnotemark[1,2]

\tnotetext[1]{This paper was funded in part by start up funds from the University of Hawai'i}

%\tnotetext[2]{The second title footnote which is a longer text matter
%   to fill through the whole text width and overflow into
%   another line in the footnotes area of the first page.}

\author[1]{Zain Jabbar}[orcid=0009-0003-2727-2313]
\cormark[1]
\ead{zjabbar@hawaii.edu}
%\ead[url]{zjabbar@hawaii.edu}

\credit{Methodology, Software, Writing}

\affiliation[1]{organization={Department of Mathematics, The University of Hawaii at Manoa},
                addressline={2565 McCarthy Mall}, 
                city={Honolulu},
                postcode={96822}, 
                state={Hawaii},
                country={USA}}

\author[2]{Andrey A. Popov}[orcid=0000-0002-7726-6224]
\cormark[2]
\ead{apopov@hawaii.edu}
%\ead[url]{apopov@hawaii.edu}
\credit{Original Idea, Methodology, Writing}

\affiliation[2]{organization={Department of Information and Computer Sciences, The University of Hawaii at Manoa},
                addressline={1680 East-West Road}, 
                city={Honolulu},
                postcode={96822}, 
                state={Hawaii},
                country={USA}}

\cortext[cor1]{Corresponding author}
\cortext[cor2]{Principal corresponding author}

\begin{abstract}
The ensemble Gaussian mixture filter (EnGMF) is a powerful, convergent particle filter capable of medium-to-high dimensional non-linear filtering. The EnGMF relies on a resampling step that can generate physically unrealistic posterior samples, that would subsequently produce physically meaningless forecasts.
This work introduces the discriminator-informed resampling procedure, that augments the posterior resampling step with a discriminator that accepts or rejects candidate particles based on their physical plausibility.
In this work these discriminators are learned through a normalizing flow approach. 
Numerical experiments on both the Ikeda map and the Lorenz '63 system show that discriminator informed resampling procedure consistently reduces error relative to the standard EnGMF in low-ensemble regimes.
\end{abstract}

\begin{graphicalabstract}
\includegraphics[width=0.9\linewidth]{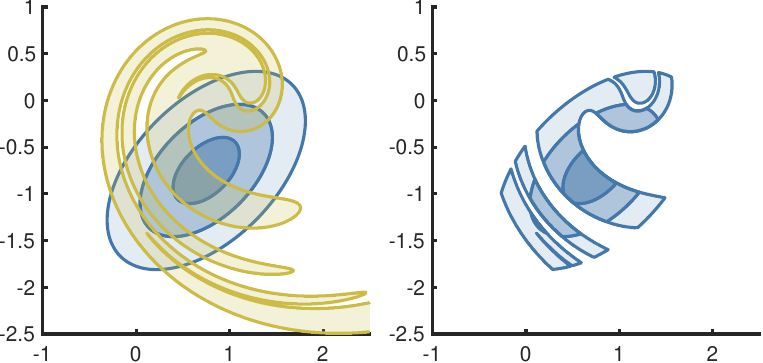}
\end{graphicalabstract}

\begin{highlights}
\item The method herein preserves the EnGMF update while rejecting non-physical particles.
\item Physicality information is incorporated directly into the Bayesian posterior.
\item Normalizing flows provide a tractable discriminator for EnGMF resampling.
\end{highlights}


\maketitle

\section{Introduction}

The ensemble Gaussian mixture filter (EnGMF)~\citep{anderson1999monte,liu2016efficient,popov2022adaptive,yun2022kernel} is a particle filtering methodology for nonlinear state estimation~\citep{sarkka2023bayesian,asch2016data,reich2015probabilistic,jazwinski2007stochastic}.
It approximates exact Bayesian filtering and converges in distribution to the exact Bayesian solution as the number of particles increases, while requiring substantially fewer particles than the standard bootstrap particle filter. The EnGMF has demonstrated success in challenging astrodynamics applications, including single-target~\citep{yun2022kernel,durant2024you,durant2023mcmc} and multi-target tracking~\citep{Reifler2023a}, where nonlinearities and highly non-Gaussian distributions are prevalent.

The EnGMF proceeds by first constructing a kernel density estimate (KDE)~\citep{silverman2018density} of the forecast ensemble, which yields a nonparametric approximation of the prior distribution. 
This prior KDE is then updated via a Gaussian sum (mixture) update~\citep{alspach1972nonlinear} to incorporate the latest observations, resulting in an analytically tractable posterior Gaussian mixture model.
A resampling step is finally performed by drawing a new ensemble from this posterior mixture, thereby realizing an updated set of particles that approximate the posterior filtering distribution.

In the special case where the system dynamics are known exactly and free of model error, the bootstrap particle filter has the desirable property that all propagated particles lie within the set of trajectories that are dynamically feasible, which we refer to as \textit{physical} realizations~\citep{gordon1993novel}. 
In contrast, the EnGMF \textit{does not} inherently preserve this property. 
Because the posterior ensemble is generated by resampling from a Gaussian mixture that is not constrained to the image of the dynamical model, the EnGMF can produce samples that cannot be realized by the true dynamics and are therefore \textit{non-physical}.
This loss of physical consistency can be particularly undesirable in applications where adherence to the underlying physics is critical, such as long-term orbit propagation or conjunction assessment.

To address this discrepancy, this work augments the EnGMF with a learned discriminator of physical plausibility.
Specifically, we construct a discriminator model based on a normalizing flow~\citep{papamakarios2021normalizing, kobyzev2020normalizing}, trained on ensembles of dynamically consistent trajectories, to approximate the distribution of physically realizable states. 
The normalizing flow provides an expressive, invertible mapping from a simple base density to the manifold of physical states, enabling efficient and exact likelihood evaluation for candidate samples. 
We then embed this discriminator within a rejection-sampling framework applied to the posterior Gaussian mixture: samples from the mixture that are assigned low likelihood under the learned flow-based density are classified as non-physical and rejected, while high-likelihood samples are retained. 
In this way we can enforce the dynamical consistency of the EnGMF ensemble without altering the core Gaussian mixture update, yielding a physically constrained variant of the EnGMF that retains its favorable statistical properties while eliminating the sampling of non-physical particles.

This paper is organized as follows:
we first go over the background about physical samples, the ensemble Gaussian mixture filter, and why resampling can be non-physical in~\cref{sec:background}.
We next provide theoretical motivation and practical implementation on how to preserve physicality in the resampling step of the EnGMF in~\cref{sec:DI-EnGMF}.
Penultimately in~\cref{sec:Ikeda-map} we showcase our methodology on the Ikeda map.
We end with some final remarks in~\cref{sec:conclusions}.

%%%%%%%%%%%%%%%%%%%%%%%%%%%%%%%%%%%%%
\section{Background and motivation}
\label{sec:background}
%%%%%%%%%%%%%%%%%%%%%%%%%%%%%%%%%%%%%

We consider the problem of quantifying uncertainty in the state $x_k$ of a chaotic dynamical system at discrete time index $k$, governed by the model,
\begin{equation}\label{eq:dynamics}
    x_k = \mathcal{M}(x_{k-1}),
\end{equation}
where $\mathcal{M}$ denotes the known dynamics.
For the purposes of this work, we assume that $\mathcal{M}$ is an exact representation of the underlying physics, i.e. we neglect model error and treat the dynamics as a perfect representation of the true evolution of the system state. 
In many applications of interest, such as astrodynamics and space situational awareness, this corresponds to the regime in which the dominant sources of uncertainty arise from initial conditions and measurements rather than from inexactness in the dynamical model.

The system is observed through noisy measurements of the form,
\begin{equation}\label{eq:measurement}
    y = h(x) + \eta, \,\, \eta \sim \mathcal{N}(0, R),
\end{equation}
where $h$ is a (potentially nonlinear) measurement operator mapping the state space to the measurement space, and $\eta$ denotes additive unbiased Gaussian noise with known covariance matrix $R$. 
The Gaussian assumption is adopted here for notational simplicity, but we note that the framework can be extended in a straightforward manner to Gaussian mixture noise as in~\citep{popov2022adaptive}. 

Let $X^-$ denote the random variable representing our prior (forecast) knowledge of the state at the current time, obtained by propagating an ensemble of initial conditions through the dynamical model $\mathcal{M}$ in~\cref{eq:dynamics}. 
Given a realized measurement $y$ from~\cref{eq:measurement}, our goal is to update this prior to a posterior that incorporates the information provided by the measurement, which in the context of Bayesian inference is typically expressed as
\begin{equation}\label{eq:old-Bayes}
    P(X^- | Y = y) \,\propto\, P(Y = y | X^-)P(X^-),
\end{equation}
where $P(X^-)$ denotes the prior distribution of our uncertainty in the state, $P(Y = y | X^-)$ is the likelihood of measuring $y$, and $P(X^- | Y = y)$, and $P(X^- | Y = y)$ is our posterior uncertainty.
For right now, we ignore any additional conditioning on other known extant information, and interpret \eqref{eq:old-Bayes} in the sense of~\citet{jaynes2003probability}.

A mathematical model such as~\eqref{eq:dynamics} is a meant to model reality. 
It is typically designed to produce \textit{meaningful predictions from meaningful inputs} that are consistent with the underlying physical reality, but it may yield \textit{meaningless outputs from meaningless inputs} that violate the physical reality.
If we have a deterministic, exact model $\mathcal{M}$, only those states that arise from valid initial conditions constitute feasible, or `physical' realizations of the system. 
Conversely states that cannot be obtained from other physical states no matter what initial condition is chosen are `non-physical'.

In many application domains this notion of physicality is both critical and problem dependent. 
In robotics, for example, the reachable set defines the subset of the space that can be attained under actuator, kinematic, and environmental constraints~\citep{bader2020maximizing,berenson2009manipulation,zacharias2007capturing}.
In aerospace applications, reachable sets characterize those states that can be achieved from a given initial condition under the governing equations of motion and admissible control inputs~\citep{natherson2024reachable,gillula2011applications,holzinger2009reachability}.
In the theory of dynamical systems, invariant sets, attractors, and related geometric structures delineate regions of state space that are preserved or recurrent under the flow of the dynamics~\citep{song2015optimization,song2022positive}. 

From this sort of perspective, it is therefore desirable to restrict our attention to posterior distributions that concentrate their mass on physically realizable states. 
That is, the particles used to represent $P(X^- | Y = y)$ should consist of inputs to the model that are physical.
Conversely, particles that are non-physical should be identified and rejected. 
In the remainder of this work, we formalize this notion of physicality, and develop mechanisms to (weakly) enforce that the EnGMF produces precisely these physical particles.

\subsection{Discriminator}

Consider a random variable $D$ that encodes our information about the physicality, as defined in the previous section, of any given state $x$. 
We denote by $\mathbb{X}\subset \mathbb{R}^n$ the set of physical states, that is, those states that are consistent with the underlying dynamical model. 
In practice, our knowledge of $\mathbb{X}$ is inexact, in the sense that we may be able to identify certain states as clearly physical or clearly non-physical, while remaining uncertain about a potentially large subset of the state space.

To formalize this notion, we introduce a discriminator function,
\begin{equation}\label{eq:discriminator}
    \mathcal{D} : \mathbb{R}^n \to \{0, 1\},
\end{equation}
which maps a state $x\in\mathbb{R}^n$ to an indicator of its physical plausibility. We interpret $D(x) = 1$ as the assertion that $x$ is physical, and $D(x)=0$ as the assertion that $x$ is non-physical. 
More generally, if we were to we relax the co-domain to $[0,1]$, $D$ could be viewed as our belief or plausibility that the state $x$ is physical. 
In other words, the discriminator summarizes our current knowledge about whether a given state is physical.

Ideally our information $D$ about physicality would be exact, in the sense that the support of $D$ would precisely be all physical states, $\Omega_D = \mathbb{X}$. 
In this case, the discriminator would reduce to the exact indicator function of $\mathbb{X}$,
\begin{equation}\label{eq:exact-discriminator}
    \mathcal{D}(x) = \begin{cases}
        1 & x \in \Omega_D\\
        0 & \text{sonst}
    \end{cases},
\end{equation}
which perfectly partitions the state space into physical and non-physical regions. 
Such a discriminator provides an unambiguous criterion for physicality. Meaning that, any state for which $D(x)=0$ c bane rejected outright as non-physical, while any state with $D(x)=1$ can be accepted as physical.

\begin{remark}[Distribution on $D$]
    Information contained in $D$ does not necessarily have a well-defined probability distribution. If $\Omega_D$ is compact then a uniform distribution on $\Omega_D$ might represent equivalent information to that described by $\mathcal{D}$ in~\cref{eq:exact-discriminator}, but this is not necessarily the case.
    It is also important to note that the distribution of $D$ would then directly depend on its parametrization~\citep{price2002uninformative}, which is not necessarily ideal.
    It is also possible to think about the reverse, that all the information $D$ is fully defined by our discriminator~\cref{eq:exact-discriminator}, and that the probability distribution of that information is of ancillary importance. 
\end{remark}

%%%%%%%%%%%%%%%%%%%%%%%%%%%%%%%%%%%%%
\subsection{Ensemble Gaussian Mixture Filter}
%%%%%%%%%%%%%%%%%%%%%%%%%%%%%%%%%%%%%

In this work, we focus on the ensemble Gaussian mixture filter (EnGMF) as an approximation scheme for Bayesian inference in~\cref{eq:old-Bayes}. 
The EnGMF belongs to the class of particle filtering methods, in the sense that it represents uncertainty by a finite ensemble of discrete state realizations which collectively approximate a (potentially) non-Gaussian probability distribution. 
In contrast to classical bootstrap particle filters, however, the EnGMF constructs an explicit Gaussian mixture approximation to the prior and posterior densities, thereby enabling analytically tractable updates while retaining the ability to capture multimodality and pronounced non-Gaussian features.
Algorithmically, the EnGMF consists of three main stages: i) construction of a kernel density estimate (KDE) of the prior distribution from the ensemble, ii) application of a Gaussian sum update to incorporate the new measurement information, and iii) resampling from the resulting posterior Gaussian mixture to obtain an updated ensemble. 

Let $E_{X^-} = [x_1^-, x_2^-, \dots x_N^-]$ denote an ensemble of $N$ particles representing our prior $X^-$. 
The prior density in the ensemble Gaussian mixture filter is approximated by a kernel density estimate (KDE),
\begin{equation}\label{eq:prior-KDE}
    p_{X^-}(x) \approx \sum_{i=1}^N \frac{1}{N} \mathcal{N}\left(x ; x_i^-, \beta^2_N \Sigma\right),
\end{equation}
using standard techniques~\citep{silverman2018density}, where $\beta^2_N$ is the Silverman bandwidth , defined later, and $\Sigma$ is the ensemble covariance, approximated by,
\begin{equation}
    \begin{aligned}
        \Sigma \approx \frac{1}{N-1}E_{X^-}\left(I_N - \frac{1}{N}1_N 1_N^T\right)E_{X^-}^T,
    \end{aligned}
\end{equation}
where $I_N$ is the identity matrix of size $N\times N$ and $1_N$ is a column vector of ones.
The KDE in~\cref{eq:prior-KDE} yields a smooth, nonparametric approximation of the prior density that becomes increasingly accurate as we increase the number of particles $N$.

The prior density estimate is then conditioned on the measurement $y$ from~\cref{eq:measurement} via a Gaussian sum update, leading to an approximate posterior density of the form,
\begin{equation}\label{eq:posterior-KDE}
    \begin{aligned}
        p_{X^+}(x) &\approx \sum_{i=1}^N w_i\,\mathcal{N}(x ; x_i^\sim, \Sigma_i^\sim),\\
        x^\sim_{i} &= x^-_i - G_i\left(h(x^-_i) - y\right),\\
        \Sigma^\sim_i &= \left(I -  G_i H_i^T\right)\beta_N^2\Sigma^-_i,\\
        G_i &= \beta^2_N\Sigma^-_i H_i^T{\left(H_i \beta^2_N\Sigma^-_i H_i^T + R\right)}^{-1},\\
        w_{i} &\,\propto\,\,  \mathcal{N}\left(y\,;\, h(x^-_i),\, H_i\beta^2_N\Sigma^-_i H_i^T + R\right),\\
        H_i &= \left.\frac{d h}{d x}\right\rvert_{x = x^-_i},
    \end{aligned}
\end{equation}
where $x^\sim_{i}$ are the posterior mixture means, $\Sigma_i^\sim$ are the covariances, and $w_i$ are the weights. The matrices $G_i$ are analogous to local Kalman gains, defined separately for each kernel component based on a linearization of the observation operator $h$ around $x^-_i$. 
We emphasize that $x^\sim_{i}$ are not themselves posterior samples, but rather the means of the Gaussian mixture components from which the posterior ensemble will subsequently be drawn.

The bandwidth factor $\beta^2_N$ in~\cref{eq:prior-KDE} controls the amount of smoothing applied to the ensemble of critical importance to the quality of the KDE approximation. 
For Gaussian target distributions and Gaussian kernels, the asymptotically optimal (in the sense of minimizing the mean integrated squared error) choice is given by the Silverman bandwidth,
\begin{equation}\label{eq:Silverman-bandwidth}
    \beta^2_N = s_\beta \left(\frac{4}{N(n + 2)}\right)^{\frac{2}{n + 4}},
\end{equation}
where $n$ is the dimension of the system, and $s_\beta$ is a scaling factor that can be tuned for more optimal performance on non-Gaussian distributions.

\begin{remark}[Optimal bandwidth scaling]
For samples drawn from a Gaussian distribution, the choice $s_\beta=1$ in~\cref{eq:Silverman-bandwidth} is asymptotically optimal.
In practice, however, ensemble-based filters frequently encounter non-Gaussian and even multimodal distributions, for which the optimal bandwidth may be smaller or larger than the Gaussian-optimal value.
In many applications of interest, taking $s_\beta<1$ can help mitigate over-smoothing, though this is largely a guideline and an ideal optimal choice is always desirable.
\end{remark}

%%%%%%%%%%%%%%%%%%%%%%%%%%%%%%%%%%%%%
\subsection{EnGMF Resampling can be non-physical}
\label{sec:naive-resampling}
%%%%%%%%%%%%%%%%%%%%%%%%%%%%%%%%%%%%%

We now turn to the resampling procedure associated with the EnGMF posterior approximation in~\cref{eq:posterior-KDE}. 
Given the Gaussian mixture representation of the posterior density, a new ensemble of size $N$ is generated by drawing samples therefrom.
The standard procedure is as follows: for each new posterior particle $x_i^+$ with index $i=1,\dots,N$,
\begin{enumerate}
    \item Draw an index $\lambda\in\{1,\dots,N\}$ from the discrete distribution defined by the mixture weights $\{w_i\}_{i=1}^N$, then,
    \item conditional on $\lambda$, draw a posterior particle $x^+_i \sim \mathcal{N}(x^\sim_\lambda, \Sigma^\sim_\lambda)$. 
\end{enumerate}
This construction yields an ensemble $E_{X^+} = [x_1^+, x_2^+, \dots x_N^+]$ that is approximately independent, up to effects from the covariances and weights themselves being dependent samples, and identically distributed from the posterior Gaussian mixture in~\cref{eq:posterior-KDE}. 
As is resampling in classical particle filters, this step replaces the weighted mixture representation with an unweighted ensemble while preserving, in expectation, the posterior density.
This can be thought of as a continuous to discrete transition~\citep{hanebeck2025ensemble}.

The main motivation of this work is that the above resampling procedure can generate posterior samples that are non-physical in the sense introduced earlier. 
The issue has two distinct aspects. 
First, each Gaussian component of the mixture has unbounded support in $\mathbb{R}^n$, so that, in principle, samples can be drawn arbitrarily far from any physically feasible region of the state space. 
Second, even when samples are drawn close to the mean, there is still no guarantee that they are physical, as it is possible for the region of physicality to be either disconnected, lower dimensional or even have fractal structure. 
The first issue can be partially mitigated by replacing Gaussian kernels with more compactly supported or otherwise better-tailored non-Gaussian kernels, such as the Epanechnikov kernel, as in~\citep{popov2024ensemble}. 
The second issue, however, is more fundamental, as it stems from the fact that the Gaussian mixture approximation itself is constructed in state space without explicit enforcement of physicality.

A straightforward, but computationally infeasible, strategy to reduce the non-physicality of the  particles is to increase the number of particles $N$. As $N\to\infty$, the Silverman bandwidth $\beta_N^2$ in~\cref{eq:Silverman-bandwidth} tends to zero, leading to kernels that approach the dirac $\delta$-distribution, which means that the support of the approximated prior approaches the outputs of the model $\mathcal{M}$ in~\cref{eq:dynamics}, which we assume is physical.
In this limit, the Gaussian mixture approximation can therefore recover the exact prior in distribution, and the Gaussian sum update would, in distribution recover the exact Bayesian posterior, thereby producing posterior particles that are physical. 
However, as this would recover the bootstrap particle filter, this eliminates the computational advantage that the EnGMF provides, therefore this strategy is not considered to be feasible.

In light of these considerations, we seek an alternative approach to generating physical samples that do not rely solely on increasing $N$. 
The approach developed in this work is to augment the standard EnGMF resampling step with a learned discriminator of physicality.
Rather than accepting all samples drawn from the posterior Gaussian mixture, we employ a discriminator, which we later implement using a normalizing flow, to assess the physicality of candidate samples.
Samples deemed non-physical according to this discriminator are rejected, yielding a modified resampling scheme that generates physically consistent particles while maintaining the efficiency of the EnGMF.

%%%%%%%%%%%%%%%%%%%%%%%%%%%%%%%%%%%%%
\section{Discriminator Informed Ensemble Gaussian Mixture Filter}
\label{sec:DI-EnGMF}
%%%%%%%%%%%%%%%%%%%%%%%%%%%%%%%%%%%%%

A discriminator can be viewed as encoding additional information about the state of the dynamical system beyond what is captured by our prior and our likelihood, but that is known to the agent performing the inference (whether it be human or machine). 
In Bayesian inference, all available and relevant information should be incorporated into the inference~\citep{jaynes2003probability}. 
Thus, if we possess any extra knowledge about the physicality of a state, such as that obtained by a discriminator, it has to be incorporated into our inference. We therefore seek a formulation of the EnGMF that explicitly accounts for this discriminator information.

Assume we are given access to information about some physicality through the random variable $D$ and its associated discriminator $\mathcal{D}$ from~\cref{eq:discriminator}. 
Bayesian inference for our prior $X^-$ with the measurement $Y=y$ in~\cref{eq:measurement} and the physicality information $D$ is given by,
\begin{equation}\label{eq:new-Bayes}
    P(X^- | Y = y, D) \,\propto\, P(Y = y | X^-, D) P(X^- | D) P(D),
\end{equation}
which augments the standard Bayesian inference~\cref{eq:old-Bayes} by conditioning explicitly on the discriminator information.
It is reasonable to assume that the information about physicality is statistically independent of both the measurement $y$ and our prior $X^-$, so that
\begin{equation}
    P(X^- | Y = y, D) \,\propto\,  P(Y = y | X^-) P(X^-) P(D).
\end{equation}
Under this independence assumption, the role of $D$ is to restrict or re-weight the support of the posterior after the standard Bayesian inference has been carried out. 
In the context of the EnGMF, this implies that the information about physicality should be applied as a correction to the Gaussian mixture posterior density obtained from the Gaussian sum update in~\cref{eq:posterior-KDE}, rather than modifying the likelihood or prior directly.

Accordingly, we can modify the posterior kernel density estimate to incorporate the discriminator as a multiplicative factor,
\begin{equation}\label{eq:new-posterior}
    p_{X^+}(x) \appropto \mathcal{D}(x)\sum_{i=1}^N w_i\,\mathcal{N}(x ; x_i^\sim, \Sigma_i^\sim),
\end{equation}
where the proportionality hides a normalizing constant ensuring that the right-hand side integrates to one. 
In this formulation, the discriminator $\mathcal{D}(x)$ acts as a mask that eliminates probability mass assigned to non-physical states, while leaving the base structure of the Gaussian mixture update intact.
In the next section, we develop a practical resampling scheme for drawing particles from~\cref{eq:new-posterior}.

%%%%%%%%%%%%%%%%%%%%%%%%%%%%%%%%%%%%%
\subsection{Resampling with a discriminator}
%%%%%%%%%%%%%%%%%%%%%%%%%%%%%%%%%%%%%

To implement sampling from~\cref{eq:new-posterior} within the EnGMF framework, a natural strategy is to augment the standard resampling step with an explicit accept–reject step conditioned by the discriminator. 
Rather than unconditionally drawing posterior particles from the posterior Gaussian mixture approximation~\cref{eq:posterior-KDE}, candidate samples are first generated by the usual EnGMF resampling procedure and then are either accepted or rejected based on their physicality.
This modification yields a sampling scheme that is consistent with the discriminator-weighted posterior while preserving the overall structure of the original filter.

Specifically, for each desired posterior ensemble member with index $i=1,\dots,N$, the discriminator-informed resampling is defined as follows:
\begin{enumerate}
    \item Draw an index $\lambda\in\{1,\dots,N\}$ from the discrete distribution defined by the posterior mixture weights $\{w_i\}_{i=1}^N$ from~\cref{eq:posterior-KDE},
    \item generate a candidate particle $x^{+}_i \sim \mathcal{N}(x^\sim_\lambda, \Sigma^\sim_\lambda)$, and,
    \item \textbf{accept the candidate $x^{+}_i$ with probability $\mathcal{D}(x^{+}_i)$, otherwise, reject the sample and repeat the procedure until acceptance.}
\end{enumerate}
Above, the Gaussian mixture posterior produced by the EnGMF serves as the proposal distribution, and the discriminator $\mathcal{D}(x)$ defines an acceptance function that eliminates probability density for non-physical candidates.
As a result, the ensemble of accepted particles is distributed according to a density proportional to $\mathcal{D}(x)$ times the EnGMF posterior mixture which is exactly what is desired in~\cref{eq:new-posterior}.
We refer to this new algorithm as the discriminator-informed ensemble Gaussian mixture filter (DI-EnGMF).

This modification enforces physicality at the resampling step, as opposed to altering the Gaussian sum update~\cref{eq:posterior-KDE} or the underlying dynamical model~\cref{eq:dynamics}.
In particular, the DI-EnGMF eliminates posterior particles that are deemed non-physical by the discriminator, thereby ensuring that meaningful inputs are passed into the dynamics, resulting in future meaningful outputs. 
At the same time, the computational complexity of the method remains comparable to that of the standard EnGMF, with the resampling step being the only step with an increase in the computational cost.

\begin{remark}[Numerical convergence and practical safeguards]
The efficiency of the accept–reject procedure depends largely on the overlap between the EnGMF posterior mixture and the support that the discriminator identifies as physical. 
When this overlap is high, the acceptance rate remains high and the additional computational overhead is minimal. 
On the other hand, if the posterior mixture is highly non-physical, repeated rejections may lead to slow convergence or, in extreme cases, it might even be numerically intractable to resample $N$ particles.
In the spirit of robustness, several practical modifications can be introduced. 
First, we may allow a small baseline acceptance probability for particles that are not in the support of the discriminator, which would also help mitigate mismatches between what is actually physical and what the discriminator deams to be physical.
Second, we can impose a maximum number of rejections per particle: once this threshold is exceeded, the current candidate particle can be either be accepted or accepted with some medium to high probability. 
In the worst possible case where all candidates are accepted regardless of the discriminator deaming them to be non-physical, the DI-EnGMF reduces to the original EnGMF resampling procedure.
\end{remark}

%%%%%%%%%%%%%%%%%%%%%%%%%%%%%%%%%%%%%%%%%%
\section{Dynamical system examples}
%%%%%%%%%%%%%%%%%%%%%%%%%%%%%%%%%%%%%%%%%%

Two dynamical systems are used as numerical examples in this work: the Ikeda map and the Lorenz '63 system. 
Both exhibit low-dimensional chaotic dynamics and possess well-defined attractors which are used as surrogates for defining physicality in a reasonable test-case.

\subsection{Ikeda map}
\label{sec:Ikeda-map}

The Ikeda map is a two-dimensional discrete-time dynamical system originally introduced by Kensuke Ikeda~\citep{ikeda1979multiple, ikeda1980optical, fm2024mapping} as a model for the behavior of light in a nonlinear optical resonator. 
It depends on a scalar bifurcation parameter $u$ and is defined by the non-linear map,
\begin{equation}\label{eq:Ikeda-map}
    \begin{aligned}
        x_{1, n+1} &= 1 + u \left[x_{1, n} \cos(t_n) - x_{2, n} \sin(t_n)\right],\\
        x_{2, n+1} &= u \left[x_{1, n} \sin(t_n) + x_{2, n} \cos(t_n)\right],\\
        &\text{with } t_{n} = 0.4 - \frac{6}{1 + x_{1, n}^2 + x_{2, n}^2},
    \end{aligned}
\end{equation}
which we assume we know exactly. 
We denote one iteration of the Ikeda map using the notation,
\begin{equation}
    \mathcal{I}(x_{1, n}, x_{2, n}) = (x_{1, n + 1}, x_{, n + 1}),
\end{equation}
where $\mathcal{I}$ is one iteration of the map.

Although the Ikeda map was originally derived in an optical context, its low dimension, nonlinear dynamics, rich bifurcation structure, sensitivity to initial conditions, and the existence of strange attractors for appropriate values of $u$ make it an ideal test case for this work.
The map is computationally cheap, invertible, and has a known easily computable form of its attractor, which we take to be physical in its context.
As the Ikeda map~\cref{eq:Ikeda-map} is also two-dimensional, it is also easy to visualize.

\subsection{Ikeda map attractor}

An attractor is a subset of the state space toward which trajectories from a surrounding neighborhood converge under the dynamics and which is invariant under the map~\citep{teschl2021ordinary}. 
For the Ikeda map, the attractor represents the region of long-term behavior which we take to be the physical subset of the state space.
Once a trajectory has approached the attractor, subsequent applications of the map~\cref{eq:Ikeda-map} remain on or near it.
This means that states far from the attractor correspond to non-physical, and states near the attractor are physical.

For the parameter value $u=0.9$, the Ikeda map exhibits chaotic behavior and a chaotic attractor.
We now describe the attractor: consider the ball $\mathcal{B}_r$ of radius $r = \sqrt{1/(1 - u)}$ centered at the origin. 
Denote by $\mathcal{I}^m$, the $m$-fold composition of the Ikeda map~\cref{eq:Ikeda-map}.
It is known~\citep{osipenko2006dynamical} that the attractor of the Ikeda map is,
\begin{equation}\label{eq:Ikeda-attractor}
    \Lambda = \bigcap_{m=0}^\infty \mathcal{I}^m(\mathcal{B}_r).
\end{equation}
We state several properties of this attractor by numerical experiment. For the Ikeda map~\cref{eq:Ikeda-map}, the Kaplan-Yorke dimension~\citep{kaplan1979functional} is approximately $1.7067$. The Minkowski–Bouligand dimension~\citep{wagon2010computational} is approximately $1.7172$. The correlation dimension~\citep{grassberger1983measuring} is estimated to be between $1.6713$ and $1.7318$.
This means that the attractor is almost, but not quite 2-dimensional, meaning that the physical points are not neatly described by some space. For the Lorenz '63 system~\cref{sec:lorenz63-system}, the Kaplan-Yorke dimension is approximately 2.06~\citep{taylor2010attractors}.

\subsection{Theory-driven numerical discriminator for the Ikeda map}
\begin{figure}
    \centering
    \includegraphics[width=0.9\linewidth]{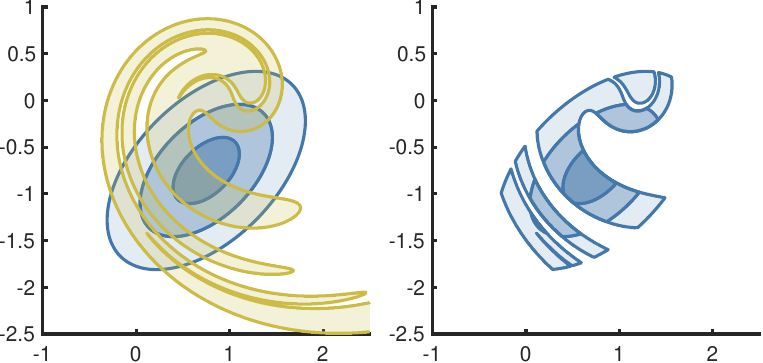}
    \caption{Left: an outline of the approximated discriminator with $m=6$ along with a 3-$\sigma$ plot of a Gaussian distribution. Right: A 3-$\sigma$ plot of the discriminator-informed Gaussian distribution showing physical regions and their relative density.}
    \label{fig:Ikeda-map-and-discriminator}
\end{figure}

As the Ikeda attractor has a known iterative solution~\cref{eq:Ikeda-attractor} it is possible to create a numerical discriminator thereof for validation purposes which we now do in this section.
The numerical discriminator that is constructed is able to determine whether a point is close to, or lies on the Ikeda attractor~\cref{eq:Ikeda-attractor}, and exploits the fact that the Ikeda map~\cref{eq:Ikeda-map} has an easily computable inverse.

We first define the inverse mapping $\mathcal{I}^{-1}$ of the Ikeda map by the following procedure,
\begin{equation}\label{eq:Ikeda-Inverse}
\begin{aligned}
    \hat{x}_{1, n} &= (x_{1, n+1} - 1)/u,\\
    \hat{x}_{2, n} &= x_{2, n+1}/u,\\
    t_n &= 0.4 - \frac{6}{1 + \hat{x}_{1, n}^2 + \hat{x}_{2, n}^2},\\
    x_{1, n} &= \hat{x}_{1, n}\cos(t_n) + \hat{x}_{2, n}\sin(t_n),\\
    x_{2, n} &= -\hat{x}_{1, n}\sin(t_n) + \hat{x}_{2, n}\cos(t_n),
    \end{aligned}
\end{equation}
which defines the inverse map as,
\begin{equation}
    \mathcal{I}^{-1}(x_{1, n+1}, x_{2, n+1}) = (x_{1, n}, x_{, n}),
\end{equation}
and constitutes an exact inverse of the forward map~\cref{eq:Ikeda-map}. 
Using this inverse, a family of discriminators $\mathcal{D}_m$ indexed by a iterate value $m$ is defined for an arbitrary point $(x_1,x_2)$ by,
\begin{equation}\label{eq:Ikeda-theory-discriminator}
    \mathcal{D}_m = \begin{cases}
        1 & \mathcal{I}^{-m}(x_1, x_2) \in \mathcal{B}_r\\
        0 & \text{otherwise}
    \end{cases},
\end{equation}
which we now prove in the limit of $m$ is a perfect discriminator of points on the attractor.

\begin{theorem}
    If  a point $(x_1, x_2)$ is on the Ikeda map attractor~\cref{eq:Ikeda-attractor}, meaning that $(x_1, x_2)\in\Lambda$, then for any $m\geq 0$,
    \begin{equation}
        \mathcal{D}_m(x_1, x_2) = 1,
    \end{equation}
    meaning that the discriminator recognizes any point on the attractor for finite $m$.
\end{theorem}
\begin{proof}
If  $(x_1, x_2)\in\Lambda$ then by definition~\cref{eq:Ikeda-attractor},
\begin{equation}
    (x_1, x_2)\in\mathcal{I}^m(\mathcal{B}_r).
\end{equation} 
as,
\begin{equation}
    \mathcal{I}^{-m}(\mathcal{I}^m(\mathcal{B}_r)) = \mathcal{B}_r,
\end{equation}
then $\mathcal{D}(x_1, x_2) = 1$ as required.
\end{proof}
\begin{corollary}
Without proof, conversely, as $m\to\infty$, the inverse map of $\{1\}$ approaches the attractor~\cref{eq:Ikeda-attractor},
    \begin{equation}
        \lim_{m\to\infty} \mathcal{D}_m^{-1}(\{1\}) = \Lambda,
    \end{equation}
    meaning that in the $m$ limit the discriminator exactly recognizes the attractor.
\end{corollary}
In practice, a finite value of $m$ can be used to construct an approximate discriminator that is computationally tractable. 
Applying this discriminator to samples drawn from a Gaussian distribution yields a highly non-Gaussian truncated distribution supported primarily on the attractor, as shown in figure~\cref{fig:Ikeda-map-and-discriminator}.
This theory-driven discriminator~\cref{eq:Ikeda-theory-discriminator} provides a baseline against which the normalizing-flow-based discriminators defined later can be evaluated.

%%%%%%%%%%%%%%%%%%%%%%%%%%%%%%%%%%%%%
\subsection{Lorenz '63}
\label{sec:lorenz63-system}
%%%%%%%%%%%%%%%%%%%%%%%%%%%%%%%%%%%%%

The Lorenz '63 system~\citep{lorenz63,lorenz1996essence} is a three-dimensional model of nonlinear dynamics given by,
\begin{equation}
\label{eq:lorenz63}
\begin{aligned}
\frac{dx}{dt} &= \sigma(y - x), \\
\frac{dy}{dt} &= x(\rho - z) - y, \\
\frac{dz}{dt} &= xy - \beta z,
\end{aligned}
\end{equation}
originally introduced as a simplified representation of thermal convection in the atmosphere.
Despite its low dimensionality, the system exhibits complex chaotic behavior and has become a canonical benchmark for the study of deterministic chaos, data assimilation, and sequential filtering~\citep{reich2015probabilistic}.

In this work, we take the  canonical parameter values \( \rho = 28 \), \( \sigma = 10 \), and \( \beta = \frac{8}{3} \). 
For this parameter regime, the Lorenz '63 equations are chaotic and thus have a corresponding chaotic attractor which we again take to be physical.
The associated attractor has a  Kaplan–Yorke dimension of approximately $2.06$~\citep{taylor2010attractors}, indicating that the long-term dynamics are slightly higher than two dimensional.

Unlike the attractor of the Ikeda map~\cref{eq:Ikeda-attractor}, simple to compute expressions for the chaotic attractor of the Lorenz '63 equations are not known, thus more complex solutions have to be utilized.

%%%%%%%%%%%%%%%%%%%%%%%%%%%%%%%%%%%%%%%%%%%%%%%%%%%%%%%
\section{Normalizing Flow Based Discriminators}
%%%%%%%%%%%%%%%%%%%%%%%%%%%%%%%%%%%%%%%%%%%%%%%%%%%%%%%

Many dynamical systems do not admit an explicit or easily computable attractors comparable to that available for the Ikeda map in~\cref{eq:Ikeda-attractor}, in particular, the Lorenz '63 equations in~\cref{sec:lorenz63-system}.
Constructing a theory-driven discriminator for its physicality is substantially more challenging: while it is straightforward to generate samples on or near the attractor through time integration, there is no nicely computable description of the attractor's support from which to draw negative examples, as in non-attractor samples.
This motivates a data-driven approach in which a model is trained directly on attractor data using one-class learning to provide a discriminator for its physicality.
Normalizing flows~\citep{papamakarios2021normalizing, kobyzev2020normalizing} are well suited to this task because they can be trained using only samples from the attractor.

A normalizing flow is a diffeomorphism $T_\theta \colon \mathbb{R}^n \to \mathbb{R}^n$ parameterized by $\theta$ such that both such that both the forward map $T_\theta$ and its inverse $T_\theta^{-1}$ are efficiently computable.
As the flow is invertible, the log-determinant of the inverse Jacobian,
\begin{equation}
    \begin{gathered}
        \log|\det J_{T_\theta^{-1}}(x)|,\\
        \text{where } J_{T_\theta^{-1}}(x) = \frac{\partial T_\theta^{-1}}{\partial x},
    \end{gathered}
\end{equation}
is also computable. Therefore, given a simple base density $p_u(u)$, such as a normal Gaussian distribution, the change-of-variables formula induces a density on $x = T_\theta(u)$ of the form,
\begin{equation}\label{eq:change-of-variables}
    p_x(x; \theta) = p_u(T_\theta^{-1}(x)) \left|\det J_{T_\theta^{-1}}(x)\right|,
\end{equation}
which is a flexible construction of potentially any density of interest, given a  sufficiently flexible $T_\theta$.

In this work, the transformation $T_\theta$ is implemented as a composition of $L$ coupling layers interleaved with permutation matrices. 
Each coupling layer uses a rational quadratic spline (RQS) bijector~\citep{durkan2019neural} with $K$ bins, parameterized by a masked conditioner network.
The $\ell$-th coupling layer partitions $x \in \mathbb{R}^n$ via binary mask $m^{(\ell)} \in \{0,1\}^n$ into masked components $x_{\text{mask}} = x \odot m^{(\ell)}$ and transformed components $x_{\text{trans}} = x \odot (1 - m^{(\ell)})$. A conditioner network $c_\phi^{(\ell)}: \mathbb{R}^{|m^{(\ell)}|} \to \mathbb{R}^{(n - |m^{(\ell)}|)(3K+1)}$ outputs spline parameters:
\begin{equation}
    \psi^{(\ell)} = c_\phi^{(\ell)}(x_{\text{mask}}),
\end{equation}
where $c_\phi^{(\ell)}$ is a multi-layer perceptron with depth $D$, hidden dimension $W$, and GELU activation. The RQS bijector applies element-wise transformations to $x_{\text{trans}}$ using parameters $\psi^{(\ell)}$, yielding the output $y_{\text{trans}}$ and log-determinant contribution $\log|\det J_{\text{RQS}}|$.

Between coupling layers, we apply PLU (permutation-lower-upper) factorization matrices $(PLU)^{(\ell)}$ to permute coordinates.
The final transformation includes a scaling factor $s \in \mathbb{R}^+$ to regulate the overall output range.
The complete forward transformation is defined as,
\begin{equation}
    T_\theta(u) = \bigcirc_{\ell = 1}^{n} (PLU)^{\ell} \circ (RQS)^{\ell}.
\end{equation}

\subsection{Training via forward invariance}

Let $\{x_i\}_{i=1}^N$ denote samples drawn from either the attractor or a set of physical states of the dynamical system of interest.
The normalizing flow parameters $\theta$ are estimated by maximum likelihood, with the following objective:
\begin{equation}\label{eq:nf-objective}
    \mathcal{L}(\theta) = \frac{1}{N}\sum_{i=1}^N \left[\log p_u(T_\theta^{-1}(x_i)) + \log\left|\det J_{T_\theta^{-1}}(x_i)\right|\right].
\end{equation}
In order to generate training data, we exploit the fact that physical inputs to a dynamical system produce physical outputs.
In the context of an attractor this is generally known as the forward invariance property: if $\Lambda$ is the attractor, and $\Phi(t, \cdot)$ are the dynamics propagating the state forward in time $t$ time units, then $\Phi(t, \Lambda) \subset \Lambda$.

For each of the Ikeda map~\cref{sec:Ikeda-map} and Lorenz '63~\cref{sec:lorenz63-system} we perform the following: an initial batch of $N=100$ points is sampled from a Gaussian distribution centered at a point known to lie near the attractor,

\begin{equation}
    x_i^{(0)}\sim \mathcal{N}(x_0, \sigma^2 I),
\end{equation}
where $x_0$ is (numerically) on the attractor and $\sigma$ is some small initial perturbation.
This batch is then evolved forward for a sufficiently long time $T_{\text{init}}=100$ to drive the ensemble onto the attractor.
For each training epoch $k$, the following steps are performed:
\begin{enumerate}
    \item Propagate the batch one step forward:
    \begin{equation}
        x_i^{(k+1)} = \Phi(\Delta t, x_i^{(k)})
    \end{equation}
    with $\Delta t = 1.0$,
    \item Evaluate the likelihood-based loss,
    \begin{equation}
        \mathcal{L}(\theta; \{x_i^{(k+1)}\}_{i=1}^N),
    \end{equation}
    via~\cref{eq:nf-objective} using the current flow parameters $\theta$,
    \item Update the parameters: 
    \begin{equation}
        \theta \leftarrow \theta - \alpha \nabla_\theta \mathcal{L},
    \end{equation} using stochastic gradient descent.
\end{enumerate}

For the Ikeda map in~\cref{sec:Ikeda-map}, $\Phi(\Delta t, x_i^{(k)})$ corresponds to $\Delta t$ iterations of the map~\cref{sec:Ikeda-map}, while for the Lorenz '63 equations in~\cref{sec:lorenz63-system}, $\Phi(\Delta t, x_i^{(k)})$ corresponds to an approximation using the Tsit5~\citep{tsitouras2011runge} Runge-Kutta solver with a constant step size of $h=0.01$.
As the number of iterations grows, $k\to\infty$, the empirical distribution of $x_i^{(k)}$ converges to the invariant measure supported on the attractor. 
The flow therefore is trained to approximate this measure.

In the numerical experiments, the initial Gaussian for the Ikeda map is chosen as,
\begin{equation}
    x_i^{(0)} \sim \mathcal{N}\left(\begin{bmatrix}
        1.25\\0.0
    \end{bmatrix}, \frac{1}{128} I_2\right),
\end{equation}
and for the For Lorenz '63 equations the initial Gaussian is chosen as, 
\begin{equation}
    x_i^{(0)} \sim \mathcal{N}\left(\begin{bmatrix}
        8\\0\\0
    \end{bmatrix}, I_3\right),
\end{equation}
where $I_2$ and $I_3$ are identity matrices of the corresponding size.
The initial points are then flowed $100$ time units or map iterations forwards in order to ensure that they are, numerically, on the attractor, are are thus deemed physical for the purposes of this work.
The base density $p_u$ is taken to be a standard Gaussian,
\begin{equation}
    p_u(x) = \mathcal{N}(x ; 0, I_n),
\end{equation}
where $n$ is the dimension of the system.

The hyperparameters of the flow are selected via a grid search, with three random initializations for each configuration.
The following ranges are explored:
\begin{itemize}
    \item Conditioner depth: $D \in \{8, 16, 32\}$
    \item Conditioner hidden dimension: $W \in \{64, 128, 256\}$
    \item RQS bins: $K \in \{4, 8, 16\}$
\end{itemize}
The number of coupling layers is fixed at $L=6$.
For each hyperparameter setting, the model is trained for $20\,000$ epochs, and the configuration yielding the lowest final testing loss $L(\theta)$ is selected, resulting in optimized parameters $\theta^*$.

Once trained, the normalizing flow defines a density $p_x(x; \theta^*)$ that approximates our information about the physicality of the system. 
This density is then used to construct a discriminator of physicality:
\begin{equation}\label{eq:nf-discriminator}
    \mathcal{D}_{\text{NF}}(x) = \begin{cases}
        1 & p_x(x; \theta^*) \geq \tau\\
        0 & p_x(x; \theta^*) < \tau
    \end{cases},
\end{equation}
where $\theta^*$ are the optimized parameters and $\tau$ is a density threshold. We set $\tau$ as the $1\%$ quantile of $\{p_x(x_i; \theta^*)\}_{i=1}^M$ for $M = 10^4$ validation samples drawn from the attractor via extended simulation, ensuring high recall on physical points while rejecting non-physical samples.

%%%%%%%%%%%%%%%%%%%%%%%%%%%%%%%%%%%%%
\section{Sequential Filtering Experiments}
%%%%%%%%%%%%%%%%%%%%%%%%%%%%%%%%%%%%%

\begin{figure}[]
\centering
\centerline{\includegraphics[width=0.8\linewidth]{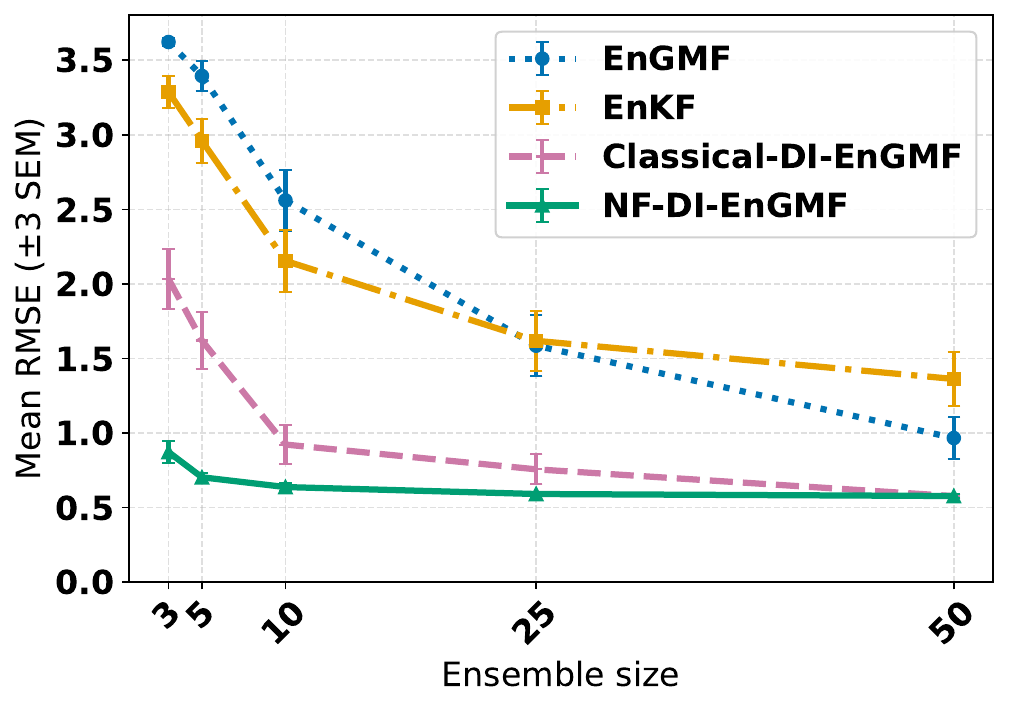}}
\caption{For the Ikeda map. For a bandwidth scaling factor~\cref{eq:Silverman-bandwidth} of $s_\beta = 1$: A comparison of the RMSE~\cref{eq:RMSE} for the DI-EnGMF (square markers and blue dotted line), EnGMF (triangle markers and orange dash-dotted line), and the EnKF (circle markers with green solid line) for ensemble sizes in the range $N=3$ to $N=20$. The value plotted is the mean over all Monte Carlo iterations and the vertical lines represent one standard deviation of RMSE along the Monte Carlo runs.}
\label{fig:rmse-plot-Ikeda}
\end{figure}

\begin{figure}[]
\centering
\includegraphics[width=0.8\linewidth]{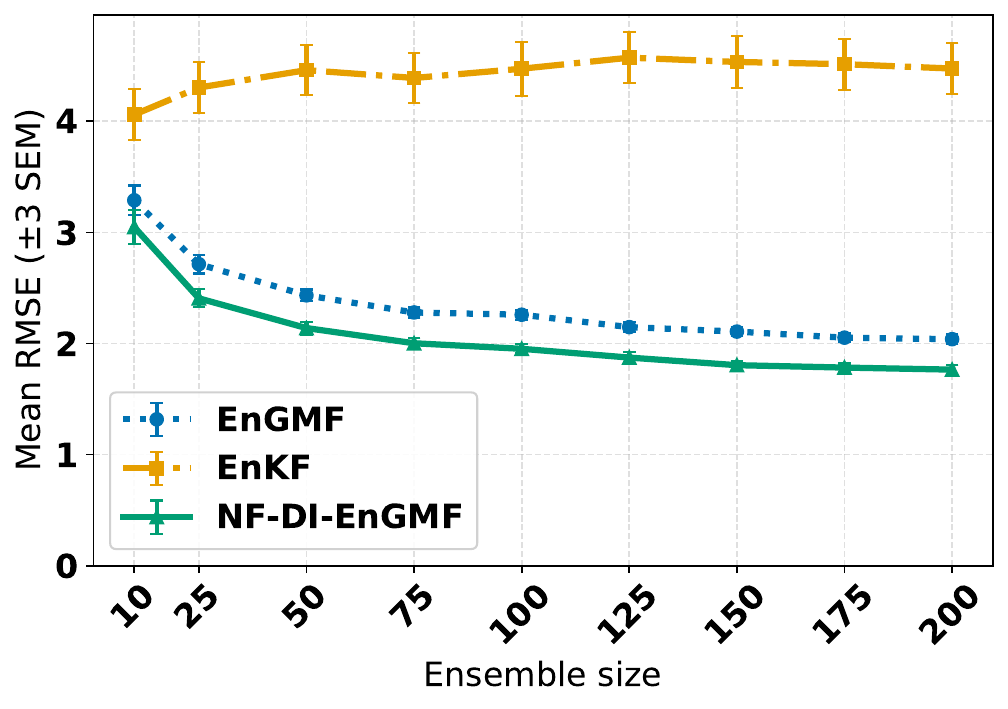}
\caption{For the Lorenz '63 equations, a comparison of the RMSE~\cref{eq:RMSE} for the DI-EnGMF (triangle markers and green solid line), EnGMF (circle markers and blue dotted line), and the EnKF (square markers with orange dash-dotted line) for ensemble sizes in the range $N=10$ to $N=200$. The value plotted is the mean over all Monte Carlo iterations and the vertical lines represent three standard errors of the mean (SEM) of the  RMSE along the Monte Carlo runs.}
\label{fig:rmse-plot-Lorenz-63}
\end{figure}

To assess the performance of the proposed method, a series of sequential filtering experiments is conducted comparing the spatial–temporal root-mean-square error (RMSE) of three algorithms: the ensemble Gaussian mixture filter (EnGMF), the ensemble Kalman filter (EnKF), and the discriminator-informed ensemble Gaussian mixture filter (DI-EnGMF), which is the main contribution of this work.
The EnKF variant considered here is the linearized ensemble Kalman filter~\citep{evensen1994sequential,burgers1998analysis,michaelson2023ensemble}, with a constant forecast ensemble inflation factor of $1.01$.

The performance metric that we use in this work is the spatio-temporal RMSE, and is defined as
\begin{equation}\label{eq:RMSE}
    \operatorname{RMSE}(x^{\text{true}}, E_{X^{+}}) = \mathbb{E}_{\operatorname{MC}}\left[\sqrt{\frac{1}{Kn} \sum_{k=1}^K \sum_{j = 1}^{n} \left(x^{\text{true}}_{j} - \left(\frac{1}{N} \sum_{i = 1}^{N} x_{j,i}^{+}\right)\right)^{2}}\right],
\end{equation}
where $n$ is the number of spatial dimensions, $K$ is the number of steps after spinup (defined later), and is computed in expectation over several Monte Carlo runs that modify the initial conditions and the measurements.

We perform two sequential filtering experiments, one with the Ikeda map in~\cref{sec:Ikeda-map} and another with the Lorenz '63 system in~\cref{sec:lorenz63-system}. 

\subsection{Ikeda Map Experiment}

In the first experiment, our prediction step is given by the noise-free Ikeda forward equations~\cref{eq:Ikeda-map} and the update steps are given by each of the filters respective update equation. We use a non-linear scalar range measurement,
\begin{equation}
    y = \sqrt{x_1^2 + x_2^2} + \eta,\quad\eta\sim\mathcal{N}(0, R),
\end{equation}
with scalar variance $R=4$. We run sequential filtering for $1100$ iterations with $100$ time steps discarded to account for filter spinup. 
We initialize our true state at,
\begin{equation}
    x_1^{\text{true}} = 1.25, x_2^{\text{true}} = 0,
\end{equation}
with the initial prior ensemble as a normally distributed family,
\begin{equation}
    E_{X^-} = x^{\text{true}} + \nu,\quad \nu \sim \mathcal{N}(0, 0.25 I_2).
\end{equation}
We look at two different discriminators for this system. First we look at what we call the classical discriminator defined by the attractor as in~\cref{eq:Ikeda-theory-discriminator}, where the number of iterations is set to $m=6$, we term the resulting algorithm the Classical-DI-EnGMF. Second we look at a normalizing flow-based discriminator as in~\cref{eq:nf-discriminator}, which we term the NF-DI-EnGMF.
The model with the lowest testing error had a conditioner depth of $D = 32$, conditioner hidden dimension of $W=256$, and RQS bins parameter of $K=4$.
The results are averaged over a total of $32$ Monte Carlo iterations, and a standard error of the mean (SEM) is also computed.

The results in for this experiment can be seen in~\cref{fig:rmse-plot-Ikeda}.
As can be seen, on this setting, the EnKF, which is a linear filter, performs quite poorly for all ensemble settings having smallest error that is just below 1.5 in terms of RMSE. 
Additionally the EnKF has very large error bars in terms of 3-$\sigma$ of SEM, meaning that the results are highly unstable.
The EnGMF performs even worse than the EnKF for very low ensemble sizes, very likely due to the fact that most samples that a Gaussian mixture of only three to ten components would generate samples that are non-physical, validating the claims presented in this work.
The Classical-DI-EnGMF is the second best performing filter in the case of small ensemble size, again validating the fact that generating more meaningful samples results in a more meaningful forecast which in turn results in physical samples, and thus less error.

What is interesting is that the normalizing flow-based discriminator performs better than that of the classical discriminator.
This is likely due to the fact that the normalizing flow is able to learn a better approximation to the true attractor than even an approximation of~\cref{eq:Ikeda-theory-discriminator} through a finite $m$. 
This likely means that the normalizing flow is as capable as a classical discriminator with a significantly higher $m$ value.
Another more interesting result is that the normalizing flow DI-EnGMF is effectively at its lowest error with just $N=5$ particles and is capable of producing great results with just $N=3$ particles. 

\subsection{Lorenz '63 Experiment}

In the second experiment, our prediction step is given by the noise-free Lorenz equations~\cref{eq:lorenz63} and the update steps are given by each of the filters respective update equation. We use a non-linear scalar range measurement,
\begin{equation}
\begin{gathered}
    y = \sqrt{\left(x_1 - 6 \sqrt{2}\right)^2 + \left(x_2 - 6\sqrt{2}\right)^2 + \left(x_3 - 27\right)^2} + \eta, \quad\eta\sim\mathcal{N}(0, R),
    \end{gathered}
\end{equation}
with scalar variance $R=4$. We run sequential filtering for $550$ iterations spaced by $\Delta t = 0.12$ time units with $50$ steps discarded to account for filter spinup. 
We initialize our true state at $(8, 0, 0)$ with the initial prior ensemble as a normally distributed family $x^{-}_{i} = x^t + \nu$ with $\nu \sim \mathcal{N}(0, I)$. 
The scaling factor of the Silverman bandwidth~\cref{eq:Silverman-bandwidth}, is set of the nominal $s_\beta = 1$.
We look at only  one discriminator for this system, which is a normalizing flow-based discriminator as in~\cref{eq:nf-discriminator}, which we again term the NF-DI-EnGMF.
The model with the lowest testing error had a conditioner depth of $D = 8$, conditioner hidden dimension of $W=256$, and RQS bins parameter of $K=8$.
The results are averaged over a total of $32$ Monte Carlo iterations, and a standard error of the mean (SEM) is also computed.

We compare how the RMSE changes according to the size of the ensemble, ranging over $N=10, \dots, 200$. The results can be seen in~\cref{fig:rmse-plot-Lorenz-63}.
As can be seen, the EnKF is not capable of producing meaningful results in this highly non-linear setting. The EnGMF, is consistently higher in error than the NF-DI-EnGMF, though not by much, meaning that while an improvement for this setting is seen, the non-physical states likely tend towards the attractor of the system exponentially fast, meaning that non-physical states rapidly become physical. 

%%%%%%%%%%%%%%%%%%%%%%%%%%%%%%%%%%%%%
\section{Conclusions}
\label{sec:conclusions}
%%%%%%%%%%%%%%%%%%%%%%%%%%%%%%%%%%%%%

This work developed a discriminator-informed ensemble Gaussian mixture filter (DI-EnGMF) that explicitly incorporates prior knowledge about the physicality of a given dynamical system, rejecting particles that are non-physical.
The discriminators in this work are learned through a normalizing flow approach.

The numerical experiments demonstrate that the DI-EnGMF  improves upon the standard EnGMF in the low-particle number regime, where sample scarcity often degrades performance.
In particular, the discriminator reduces particles that are non-physical, leading to lower errors for a given particle number. 
This improvement has direct computational implications in applications where the forecast model is expensive: achieving a target level of accuracy with fewer ensemble members translates into substantial savings in computational cost.
Moreover, by preserving the asymptotic convergence properties of the EnGMF, the DI-EnGMF accelerates convergence toward the true Bayesian posterior as the particle number increases.

Future work will focus on extending this framework to higher-dimensional and more complex systems where the geometry of physically admissible regions is harder to characterize.
A natural next step is the application to the Lorenz '96 equations~\citep{lorenz1996predictability}  and to multi-scale geophysical systems, where both strong nonlinearity and high dimensionality pose significant challenges to particle filtering. 
Preliminary experiments suggest that a straightforward transplantation of the present approach to Lorenz '96 does not yield substantial error reductions, meaning that additional methodological improvements need to be developed.
These may include more expressive or problem-adapted discriminators, tighter integration between the discriminator and the proposal mechanism, or hybrid strategies that also take advantage of localization and covariance regularization techniques.

More broadly, an important research direction is the development of methods for learning physicality discriminators when no explicit convergent test for physicality is available. 
Normalizing flows, used herein, constitute one such avenue, but there but other generative~\citep{goodfellow2014generative} and discriminative architectures, as well as variational autoencoders~\citep{kingma2013auto}. 
Ultimately, the goal is to construct filtering algorithms that are not only consistent and computationally efficient, but also aligned with the underlying physics, even in complex, high-dimensional applications.

\section*{Acknowledgments}
A large language model interface, Perplixity.ai was used for improving language and clarity during the editing process set on ``best model'' mode.

\printcredits

%% Loading bibliography style file
%\bibliographystyle{model1-num-names}
\bibliographystyle{cas-model2-names}

% Loading bibliography database
\bibliography{bibfiles/banana,bibfiles/covarianceshrinkage,bibfiles/em,bibfiles/engmf,bibfiles/filteringgeneral,bibfiles/kernelapproximation,bibfiles/misc,bibfiles/multifidelity, bibfiles/enkf, bibfiles/discriminator_engmf}

%\vskip3pt
\newpage
\bio{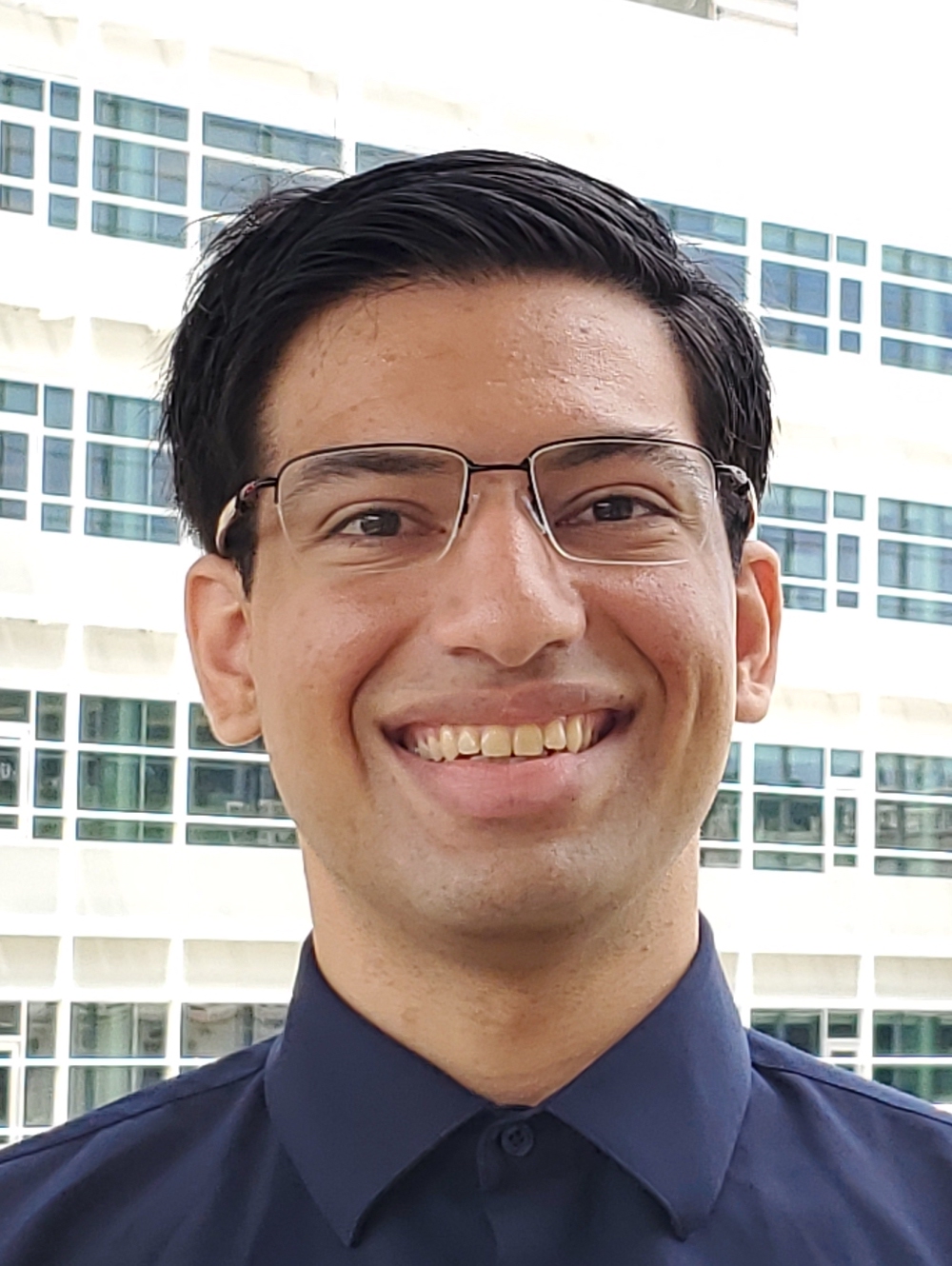}
Zain Jabbar received his B.S. in Mathematics from the University of Hawai'i at M\=anoa, Honolulu, HI.
He is currently a Ph.D. candidate in the Department of Mathematics at the University of Hawai'i at M\=anoa, advised by Dr.\ Andrey A.\ Popov.
His research interests include stochastic filtering, machine learning, and multi-target tracking.
\endbio

\vskip3pc

\bio{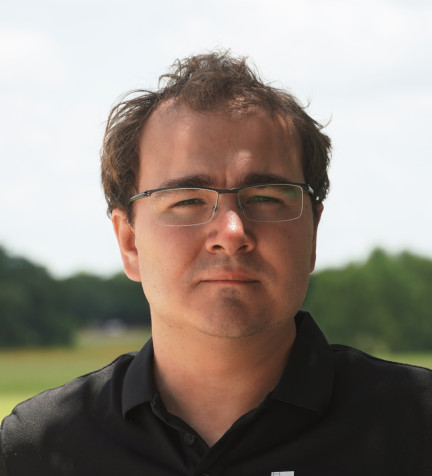}
Andrey A. Popov received his B.S. in Mathematics from RPI, Troy, NY, and his Ph.D. in Computer Science from Virginia Tech, and was a postdoctoral fellow at the Oden Institute for Computational Engineering and Sciences at the University of Texas at Austin.
He is currently an Assistant Professor at the Department of Information \& Computer Sciences at the University of Hawai'i at Manoa.
He has authored over 44 peer-reviewed papers in journals an conferences including in SIAM SISC and FUSION.
In 2024 he was awarded the Jean-Pierre Le Cadre prize for the best paper at the International Conference on Information Fusion (FUSION).
His primary research focuses on fusing theory-driven and data-driven methods for scientific applications, data assimilation, and aerospace problems. His other research interests include directional statistics, and geographic information systems.
\endbio

\end{document}